\documentclass[letterpaper, 10 pt, conference]{ieeeconf}  

\IEEEoverridecommandlockouts                              

\usepackage{amsmath,amsfonts,amssymb,amsthm}
\usepackage{algorithm}
\usepackage{algpseudocode}
\usepackage{array}
\usepackage[caption=false,font=normalsize,labelfont=sf,textfont=sf]{subfig}
\usepackage{textcomp}
\usepackage{url}
\usepackage{verbatim}
\usepackage{graphicx}
\def\BibTeX{{\rm B\kern-.05em{\sc i\kern-.025em b}\kern-.08em
    T\kern-.1667em\lower.7ex\hbox{E}\kern-.125emX}}
\usepackage{balance}
\usepackage{hyperref}
\usepackage{xcolor}
\usepackage{multirow}
\usepackage{mathtools}
\usepackage{color,soul}
\usepackage{amsthm}
\usepackage{tikz}
\usetikzlibrary{arrows.meta, positioning, shapes.geometric, shapes.multipart, calc, fit, backgrounds}

\tikzset{arrow/.style={-Latex, line width=0.5pt}}

\tikzset{
  block/.style={
    draw, thick, rounded corners=1.5mm,
    align=center,
    minimum width=16mm,      
    minimum height=7mm,      
    inner xsep=2mm,          
    inner ysep=1.5mm,        
    font=\sffamily\footnotesize  
  },
  line/.style={-Latex, thick}
}

\tikzset{
  dashedbox/.style={
    draw, dashed, rounded corners=3mm, thick,
    inner sep=6pt
  }
}

\theoremstyle{plain}
\newtheorem{theorem}{Theorem}
\newtheorem{proposition}{Proposition}

\theoremstyle{definition}
\newtheorem{definition}{Definition}
\theoremstyle{remark}
\newtheorem{remark}{Remark}

\title{\LARGE \bf
Airspeed Forward-Invariance for Unpowered Fixed-Wing Aircraft
}

\author{H. Emre Tekaslan$^{1}$ and Ella M. Atkins $^{2}$
\thanks{$^{1}$H. Emre Tekaslan is with the Kevin T. Crofton Department of Aerospace and Ocean Engineering,
        Virginia Polytechnic Institute and State University, Blacksburg, VA 24060, USA,
        {Corresponding author: \tt\small tekaslan@vt.edu}}%
\thanks{$^{2}$Ella M. Atkins is with the Kevin T. Crofton Department of Aerospace and Ocean Engineering,
        Virginia Polytechnic Institute and State University,
        Blacksburg, VA 24060, USA,
        {\tt\small ematkins@vt.edu}}%
}

\begin{document}

\maketitle
\thispagestyle{empty}
\pagestyle{empty}

\begin{abstract}
Autonomous fixed-wing flight is becoming a key capability in aerial robotics, enabling sensing, mobility, and contingency operations across both small-scale Uncrewed Aircraft Systems and large-scale Advanced Air Mobility. During unpowered operation in fixed-wing platforms, airspeed is regulated solely through potential–kinetic energy exchange, making airspeed dynamics highly sensitive to guidance commands, particularly under wind. This paper presents a viability-based airspeed protection for ground-referenced guidance in steady wind, where airspeed evolution depends explicitly on the commanded flight path angle. Leveraging Nagumo’s tangency condition, we derive a closed-form, wind-dependent characterization of admissible guidance commands that guarantees forward invariance of a safe airspeed envelope. These conditions are embedded within an offline quadratic programming framework to certify airspeed-safe maneuver primitives for non-ascending flight at the guidance level. The approach is validated using a high-fidelity unpowered fixed-wing aircraft model on gliding trajectories formed by concatenating certified maneuver primitives, demonstrating strict airspeed boundedness. Future work will address unsteady wind fields and flight experiments.
\end{abstract}
\section{Introduction}
\label{sec:intro}
Uncrewed Aircraft Systems (UAS) are increasingly deployed in aerial robotics applications, including small UAS for surveillance and robot team sensing \cite{fi13070174, 8979150, 7914761}, as well as Advanced Air Mobility (AAM) platforms that may rely on gliding modes to extend range or execute safe low-energy landings \cite{akinola2025markov}. In both settings, safe autonomous operation critically depends on maintaining aerodynamic feasibility throughout the mission. Maintaining airspeed within prescribed bounds is therefore fundamental: insufficient airspeed compromises lift generation and controllability, while excessive airspeed may violate structural or operational limits \cite{Raymer2012}. For unpowered or power-degraded fixed-wing autonomous UAS—such as engine-out emergency landing scenarios \cite{tekaslan_search}, ranging from lightweight deployable gliders \cite{Kahn2019, 7152281} to AAM vehicles comparable in size to light general aviation—airspeed regulation becomes an explicit energy-management problem. In these platforms, wind can constitute a nontrivial component of total airspeed, making it critical to distinguish between ground-referenced guidance and air-relative motion. In the absence of propulsion, airspeed evolves solely through gravitational potential energy exchange and aerodynamic dissipation, rendering it highly sensitive to guidance decisions that shape flight path geometry. Consequently, guidance commands that are kinematically feasible in the ground frame may nevertheless induce unsafe airspeed excursions, particularly under wind. Because commands such as flight path angle or course changes are typically specified in ground coordinates, even time-invariant guidance can produce nontrivial airspeed dynamics depending on wind magnitude, direction, and maneuver geometry. This coupling renders autonomous unpowered UAS guidance a safety-critical aerial robotics problem requiring formal feasibility guarantees.

Viability theory provides a rigorous framework for reasoning about forward invariance of constraint sets in dynamical systems \cite{Aubin1991}. Nagumo’s tangency condition gives a necessary and sufficient local criterion for forward invariance by requiring the system’s vector field to point inward at the boundary of the constraint set. Viability-based methods have been applied primarily to compute or approximate viability kernels in state space systems using discretization or set-based techniques \cite{liu2019flexibility, bouguerra2018viability, muhammad:hal-01143861}, particularly for robot collision avoidance. Control Barrier Functions (CBFs) \cite{cbf,cbf2} reinterpret Nagumo’s condition as inequality constraints on control inputs, enabling real-time safety enforcement through online optimization \cite{ames2021} or graph search \cite{8910394}. CBFs have been integrated into motion planning and safety-critical control frameworks for collision avoidance \cite{ahmad2022adaptive, manjunath2021safe, mestres2025safe, thirugnanam2022safety}, flight envelope protection \cite{autenrieb2025quadratic, mesbah2024}, and autonomous landing on a moving platform \cite{9507289}.

These approaches primarily enforce safety at the control level through online CBF-constrained quadratic programs or reactive safety filters, and do not explicitly characterize a priori which guidance-level maneuvers are dynamically admissible. For unpowered UAS, the coupling between ground-referenced guidance and wind implies that while airspeed safety can be enforced at the control level, as demonstrated in prior flight envelope protection studies \cite{autenrieb2025quadratic, mesbah2024} using CBFs, safety has not been explicitly characterized at the guidance or planning level. This distinction is critical for anytime and real-time path planning, where candidate plans must be dynamically executable as generated without reliance on computationally expensive post-processing, online simulation, or control modification.

A prior work \cite{tekaslan2026feasibility} explored guidance feasibility under wind but did not provide formal guarantees. This paper addresses these gaps by formulating airspeed safety for unpowered fixed-wing UAS as a guidance-level viability problem under steady horizontal wind. Rather than enforcing safety through online control filtering, Nagumo’s tangency condition is applied directly to the airspeed dynamics induced by ground-referenced guidance commands. This yields a closed-form, wind-dependent characterization of admissible ground-referenced flight path angle commands that guarantees forward invariance of a prescribed airspeed envelope. These conditions are embedded within an offline optimization framework to synthesize airspeed-safe discrete maneuvers that can be queried by a real-time path planner.

This paper offers two contributions.  First, a viability analysis is performed to derive wind-dependent admissible sets of ground-referenced flight path angle commands guaranteeing airspeed boundedness for unpowered fixed-wing UAS. Second, airspeed-safe discrete maneuver primitives are synthesized offline to enable direct integration into real-time and anytime path planning without the need for online feasibility repair or control modification.

The remainder of the paper is organized as follows. Section~\ref{sec:prelim} reviews unpowered UAS airspeed dynamics and the coupling between ground-referenced guidance and airspeed under wind. Section~\ref{sec:method} develops a viability-theoretic forward invariance analysis and embeds the resulting conditions into an optimal guidance formulation. Section~\ref{sec:numeric} details numerical implementation and computational considerations. Section~\ref{sec:application} applies the framework to a six-degree-of-freedom (6-DoF) unpowered Cessna~182 model and evaluates viability-constrained guidance solutions. Finally, the guidance policy is integrated with a path planner and validated through high-fidelity dynamic simulations.
\section{Unpowered Fixed-wing UAS Model}
\label{sec:prelim}
This paper models an unpowered fixed-wing UAS as a point mass subject to aerodynamic lift and drag, as shown in Figure \ref{fig:fbd}, operating in a steady, uniform horizontal wind. Let $v_a \in \mathbb{R}^+$ denote the airspeed magnitude, defined as the norm of the air-relative velocity vector $\mathbf{v}_a \in \mathbb{R}^3$. The air-relative and ground-referenced flight path angles are respectively denoted by $\gamma_a, \gamma_g \in \Gamma$, measured negative downward from the local horizontal where $\Gamma \subseteq \left ( -\frac{\pi}{2}, \frac{\pi}{2} \right)$ defines an admissible flight path angle set. Aircraft mass is $m\in \mathbb{R}^+$, $g\in \mathbb{R}^+$ is gravitational acceleration, and weight is $W = mg$.
\begin{figure}[t!]
    \centering
    \includegraphics[width=.8\linewidth]{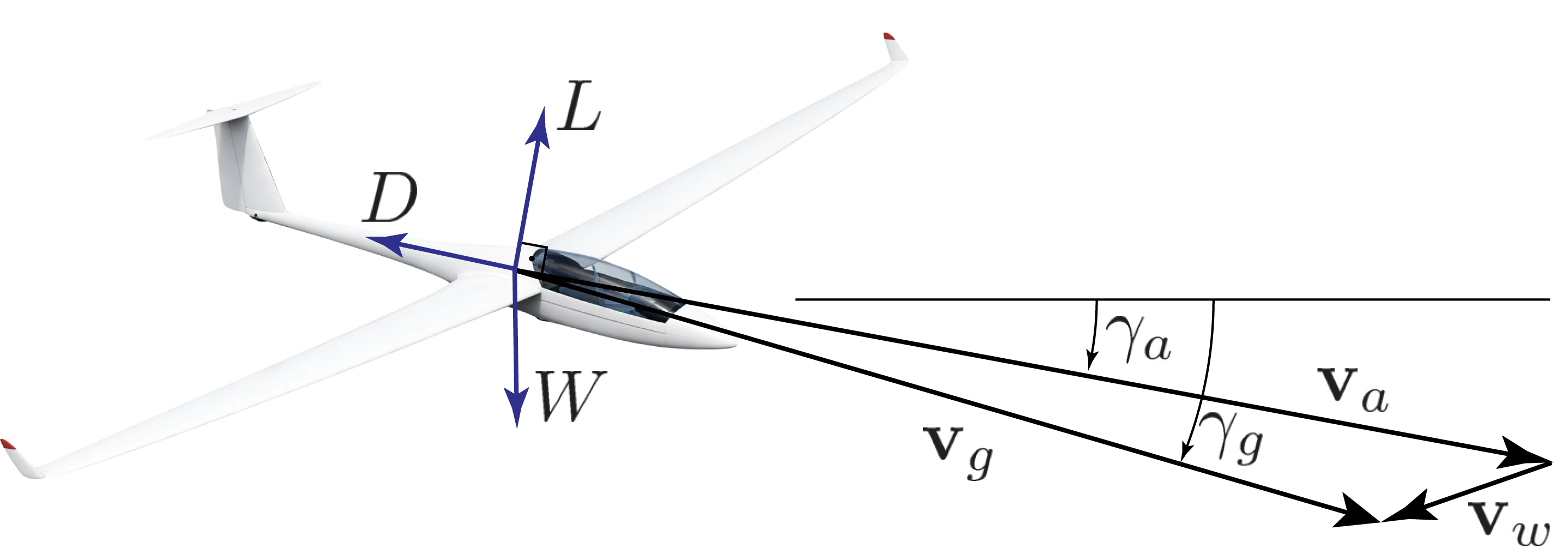}
    \caption{Fundamental forces and velocity vectors of an unpowered fixed-wing UAS shown in the vertical plane.}
    \label{fig:fbd}
\end{figure}

\subsection{Airspeed Dynamics}
The longitudinal dynamics of airspeed are obtained by projecting forces along the air-relative velocity direction. For an unpowered aircraft, the airspeed dynamics are given by
\begin{equation}
m \dot{v}_a = -D(v_a,\gamma_a) - W\sin \gamma_a,
\label{eq:airspeed_dyn}
\end{equation}
where $D$ denotes aerodynamic drag and can be modeled using a quadratic drag polar \cite{Raymer2012},
\begin{equation}
D = \frac{1}{2} \rho S v_a^2 \left( C_{D0} + k C_L^2 \right),
\label{eq:drag_polar}
\end{equation}
where \(\rho\) is air density, \(S\) is the wing planform area, \(C_{D0}\) is the zero-lift drag coefficient, \(k\) is the lift-induced drag factor, and \(C_L\) is the lift coefficient. In a coordinated turn at constant bank angle $\mu \in \mathbb{R}$, lift balances the component of weight normal to the flight path,
\begin{equation}
L = W\frac{\cos \gamma_a}{\cos \mu}.
\label{eq:lift_balance}
\end{equation}
Using the lift relation \(L = \tfrac{1}{2} \rho S v_a^2 C_L\) \cite{Raymer2012}, the lift coefficient can be expressed as
\begin{equation}
C_L = \frac{2W}{\rho S v_a^2} \frac{\cos \gamma_a}{\cos \mu}.
\label{eq:CL}
\end{equation}
Substituting (\ref{eq:CL}) into (\ref{eq:drag_polar}) yields drag as a function of airspeed and air-relative flight path angle for a given bank angle,
\begin{equation}
D(v_a,\gamma_a)
= \tfrac{1}{2} \rho S C_{D0} v_a^2
+ \frac{2 k W^2}{\rho S v_a^2}
\left( \frac{\cos \gamma_a}{\cos \mu} \right)^2.
\label{eq:drag_explicit}
\end{equation}
Dividing (\ref{eq:airspeed_dyn}) by \(m\), the airspeed acceleration $\dot{v}_a \in \mathbb{R}$ can be written as:
\begin{equation}
\dot{v}_a
= -\frac{D(v_a,\gamma_a)}{m} - g \sin \gamma_a.
\label{eq:va_dot}
\end{equation}
Equation (\ref{eq:va_dot}) shows that airspeed evolution is governed by the balance between gravity-induced acceleration along the flight path and aerodynamic dissipation, both of which depend on air-relative flight path angle \(\gamma_a\).

\subsection{Ground-Referenced Guidance and Wind Coupling}
\label{sec:wind-coupling}
Let \(\mathbf{v}_g \in \mathbb{R}^3\) denote the ground-relative velocity vector and \(\mathbf{v}_w \in \mathbb{R}^3\) the wind velocity vector, assumed constant and horizontal. The airspeed satisfies
\begin{equation}
\mathbf{v}_a = \mathbf{v}_g - \mathbf{v}_w.
\label{eq:velocity_relation}
\end{equation}
The ground-referenced flight path angle \(\gamma_g\) is defined as
\begin{equation}
\gamma_g = \sin^{-1}\!\left( \frac{v_{g,z}}{\lvert \mathbf{v}_g \rvert} \right),
\label{eq:gamma_g}
\end{equation}
where \(v_{g,z}\) is the vertical component of the ground-relative velocity. Air-relative flight path angle \(\gamma_a\) is given by
\begin{equation}
\gamma_a = \sin^{-1}\!\left( \frac{v_{a,z}}{\lvert \mathbf{v}_a \rvert} \right).
\label{eq:gamma_a}
\end{equation}
With a horizontal steady wind assumption, the vertical velocity components satisfy $v_{a,z} = v_{g,z}$. However, the magnitudes $v_a$ and $v_g$ differ due to wind, implying that $\gamma_a \neq \gamma_g$ in general. Consequently, commanding a ground-referenced flight path angle $\gamma_g$ implicitly determines the air-relative flight path angle $\gamma_a$ through (\ref{eq:velocity_relation}). We can also use direct mapping $\mathcal{G}:\Gamma\rightarrow\Gamma$ from $\gamma_a$ to $\gamma_g$ per \eqref{eq:gamma_g2a} \cite{beard_mclain}.
\begin{equation}
    \gamma_g = \mathcal{G}(\gamma_a;\mathbf{v}_a, \mathbf{v}_w) = \sin^{-1}\left( \frac{v_a\sin\gamma_a - v_{w,z}}{\lvert \mathbf{v}_a + \mathbf{v}_w  \rvert} \right)
    \label{eq:gamma_g2a}
\end{equation}

Since \(\gamma_a\) enters the airspeed dynamics (\ref{eq:va_dot}) both through the gravitational term and through the drag force, ground-referenced guidance directly influences airspeed evolution despite not regulating it explicitly. In the presence of wind, even time-invariant commands can therefore induce airspeed acceleration, depending on wind magnitude, direction, and maneuver geometry.

\section{Methodology}
\label{sec:method}
This section first develops the viability analysis on a reduced-order point-mass fixed-wing UAS model for which the airspeed dynamics admit an explicit dependence on the ground-referenced flight path angle. This structure enables derivation of a closed-form viable set using Nagumo’s tangency condition, providing an interpretable and analytical characterization of admissible guidance commands. This approach provides a mathematically rigorous basis for certifying airspeed safety under wind-coupled guidance laws. Subsequently, the same Nagumo-based construction is then applied to a nonlinear UAS model. In this case, a closed-form expression for the viable set is no longer available due to the increased model complexity. Instead, the tangency condition is incorporated directly into an optimal control problem as nonlinear constraints. The resulting guidance law is selected to satisfy Nagumo’s condition for the nonlinear dynamics, yielding a numerically characterized viable optimal guidance policy without relying on an explicit analytical solution.

\subsection{Airspeed Viability Set}
Let $f:\Omega \rightarrow \mathbb{R}$ be a shorthand notation for the airspeed dynamics per \eqref{eq:va_dot} such that $f = \dot{v}_a$ where $\Omega \triangleq \mathbb{R^+}\times\Gamma \subset \mathbb{R}^2$. Also, let $\mathcal{V} \subset \mathbb{R}^+$ be a compact admissible airspeed set \eqref{eq:Vset}:
\begin{equation}
\mathcal{V} \triangleq \{ v_a \in \mathbb{R}^+ \mid 0 < v_{\min} \le v_a \le v_{\max} < \infty \}.
\label{eq:Vset}
\end{equation}
The objective is to ensure that trajectories initialized in $\mathcal{V}$ remain in $\mathcal{V}$ for all future time. To satisfy this goal rigorously, we must characterize directions that keep airspeed inside $\mathcal{V}$ at its boundary. This leads us to introduce forward invariance of $\mathcal{V}$ that is characterized using the contingent cone, also known as the  Bouligand tangent cone \cite{Aubin1991}.
\begin{definition}[Contingent Cone]
    The contingent cone to \(\mathcal{V}\) at a point \(v_a \in \mathcal{V}\), denoted \(T_{\mathcal{V}}(v_a)\), is defined as
    \begin{equation}
    T_{\mathcal{V}}(v_a)
    \triangleq
    \left\{
    \eta \in \mathbb{R}
    \middle|
    \lim_{h \rightarrow 0^+} \inf
    \frac{d_\mathcal{V}(v_a + h\eta)}{h}
    = 0
    \right\},
    \label{eq:contingent_cone}
    \end{equation}
    where $\eta \in \mathbb{R}$ is a candidate direction, $h>0$ is a scalar step size, and $d_\mathcal{V}(x)$ expresses the distance of $x \in \mathbb{R}$ to $\mathcal{V}$ as:
    \begin{equation}
        d_{\mathcal{V}}(x) = \inf_{z \in \mathcal{V}}\; \lvert x - z\rvert.
    \end{equation}
    For the interval $\mathcal{V}$, the contingent cone takes the form
    \begin{equation}
    T_{\mathcal{V}}(v_a) =
    \begin{cases}
    \mathbb{R}, & v_a \in (v_{\min}, v_{\max})\\
    [0, +\infty), & v_a = v_{\min} \\
    (-\infty, 0], & v_a = v_{\max}
    \end{cases}
    .
    \label{eq:contingent_cone_interval}
    \end{equation}
\end{definition}
In simple terms, the contingent cone is the set of admissible instantaneous airspeed rates that keep the airspeed inside $\mathcal{V}$. We now define a special case of $\mathcal{V}$.
\begin{definition}[Viability Domain]
    $\mathcal{V}$ is a viability domain (i.e., forward invariant set) of the map $f: \Omega \rightarrow \mathbb{R}$ if
    \begin{equation}
        f(v_a,\gamma_a) \in T_{\mathcal{V}}(v_a), \quad \forall v_a \in \mathcal{V}.
        \label{eq:nagumo_condition}
    \end{equation}
\end{definition}
The Nagumo theorem provides a necessary and sufficient condition \cite{Aubin1991} for viability of \(\mathcal{V}\) under $f$:
\begin{theorem}
    Given locally compact $\mathcal{V}$ and continuous dynamics $\dot{v}_a = f(v_a,\gamma_a)$ from $\Omega$ to $\mathbb{R}$, $\mathcal{V}$ is locally viable under $f$ if and only if $\mathcal{V}$ is a viability domain of $f$.
\end{theorem}

In the context of ground-referenced guidance, the vector field \(f(v_a,\gamma_a)\) is not directly specified but is induced by guidance commands of $\gamma_g$ per \eqref{eq:gamma_g2a} through wind-dependent relationships. Thus, verifying \eqref{eq:nagumo_condition} requires characterizing how autopilot reference $\gamma_g$ for path tracking influences the sign and magnitude of $\dot v_a = f(v_a,\gamma_a)$ at the boundary of $\mathcal{V}$. To make the Nagumo tangency condition \eqref{eq:nagumo_condition} explicit, we now specialize the airspeed dynamics at the boundary of $\mathcal V$. From \eqref{eq:contingent_cone_interval}, forward invariance of $\mathcal V$ requires,
\begin{equation}
f(v_{\max},\gamma_a) \leq 0 \leq f(v_{\min},\gamma_a), \quad \forall \gamma_g \in \Gamma_{\mathcal{V}} \subseteq \Gamma,
\label{eq:nagumo_boundary_conditions}
\end{equation}
to keep airspeed acceleration pointing inward of the viability domain $\mathcal{V}$ where $\Gamma_{\mathcal{V}}$ is the set of viable guidance commands. 

The contingent cone characterizes admissible instantaneous airspeed rates independently of the applied guidance, whereas the guidance law determines whether the induced airspeed dynamics satisfy this tangency condition. Consequently, an autonomous unpowered UAS must be cognizant of a set of ground-referenced flight path angle commands that guarantee airspeed boundedness, enabling both viability-constrained trajectory planning and real-time guidance. We now formally define the corresponding set of viable ground-referenced flight path angles.
\begin{definition}[Viable Ground-Referenced Flight Path Angle Set]
The set of ground-referenced flight path angle commands that satisfy the Nagumo tangency condition, as a function of the airspeed and wind velocity vectors, is called viable and defined as:
\begin{equation}
\begin{aligned}
\Gamma_{\mathcal V}(\mathbf v_a,\mathbf v_w)
\triangleq
\{
\gamma_g = \mathcal G(\gamma_a;\mathbf v_a,\mathbf v_w) \lvert&
f(v_a,\gamma_a) \in T_{\mathcal V}(v_a),\\
& v_a \in \mathcal V\}.
\end{aligned}
\label{eq:viable_gamma_g_set}
\end{equation}
\end{definition}

One can substitute airspeed dynamics \eqref{eq:va_dot} into \eqref{eq:nagumo_boundary_conditions} to obtain closed-form expressions for viable air-relative flight path angle $\gamma_a$:
\begin{equation}
    \frac{-D(v_{\max},\gamma_a)}{m} - g\sin\gamma_a \leq 0 \leq \frac{-D(v_{\min},\gamma_a)}{m} - g\sin\gamma_a.
    \label{eq:viable_gamma_a_interval}
\end{equation}
Rearranging the inequality, one obtains:
\begin{equation}
    \underbrace{\sin^{-1}\left(\frac{-D(v_{\max},\gamma_a)}{W}\right)}_{\underline{\gamma_a}} \leq \gamma_a \leq \underbrace{\sin^{-1}\left(\frac{-D(v_{\min},\gamma_a)}{W}\right)}_{\overline{\gamma_a}}.
    \label{eq:viable_gamma_a_interval2}
\end{equation}
The mapping $\mathcal{G}$ is monotonically increasing in $\gamma_a$ for constant $\mathbf{v}_a$ and $\mathbf{v}_w$. So the viable air-relative flight path angle set in \eqref{eq:viable_gamma_a_interval2} can be projected onto the viable ground-referenced flight path angle set $\Gamma_{\mathcal{V}}$ through $\mathcal{G}$:
\begin{equation}
    \Gamma_{\mathcal V}(\mathbf{v}_a,\mathbf{v}_w) = \left\{\mathcal{G}(\underline{\gamma_a};\mathbf{v}_a, \mathbf{v}_w) \leq \gamma_g \leq \mathcal{G}( \overline{\gamma_a};\mathbf{v}_a, \mathbf{v}_w)\right\}.
    \label{eq:viable_gamma_g_set2}
\end{equation}
Thus, \eqref{eq:viable_gamma_g_set2} converts the abstract Nagumo tangency condition into an explicit interval constraint on the flight path angle, making airspeed boundedness directly enforceable.

In a coordinated turn under wind, the direction of the airspeed vector $\mathbf{v}_a$ varies continuously while its magnitude remains approximately constant, whereas the wind velocity $\mathbf{v}_w$ is assumed fixed in the inertial frame. As a result, the air-to-ground mapping $\mathcal{G}(\mathbf{v}_a(t),\mathbf{v}_w,\gamma_a)$ becomes time-varying during maneuvering flight. Consequently, the viable ground-referenced flight path angle set $\Gamma_\mathcal{V}$ in \eqref{eq:viable_gamma_g_set2} cannot be enforced as a static constraint exploited in trajectory planning and guidance, and must instead be incorporated within a dynamic, viability-constrained guidance framework. In other words, for any admissible maneuver initialized at an airspeed $v_a \in \mathcal V$ and subject to a steady wind field, there exists a uniquely defined viable guidance set $\Gamma_{\mathcal V}(\mathbf v_a,\mathbf v_w)$ that guarantees forward invariance of the airspeed envelope.

\subsection{Viability-Constrained Guidance Optimization}
\label{sec:viability_ocp}
As discussed above, an unpowered fixed-wing UAS can increase airspeed only by exchanging potential energy for kinetic energy through a steeper flight path angle. It is therefore necessary to identify the set of ground-referenced flight path angle commands $\gamma_g$ that maintain $v_a \in \mathcal{V}$ and guarantee forward invariance of the airspeed envelope, even arbitrarily close to its boundaries. To this end, this section develops a viability-based framework for synthesizing guidance policies that preserve airspeed boundedness during turning maneuvers in steady horizontal wind. The resulting time-varying viability-constrained guidance problem is approximated via a parametric optimization formulation, enabling tractable computation of admissible commands. We begin by defining a parametric maneuver.
\begin{definition}[Maneuver Parameter Space]
Let the maneuver space be defined as $\Sigma \triangleq \mathcal{X} \times \mathbb{R}^3$, where $\Delta\chi \in \mathcal{X} \subseteq [-2\pi,2\pi]$ denotes a commanded course change and $\mathbf v_w \in \mathbb{R}^3$ denotes a steady inertial wind vector. A constant magnitude airspeed turn maneuver is an element in $\Sigma$ as
\begin{equation}
\sigma \triangleq (\Delta\chi,\mathbf{v}_w;\mathbf{v}_a) \in \Sigma.
\end{equation}
\end{definition}
The maneuver parameters specify the kinematic conditions under which guidance must be generated, but do not by themselves ensure airspeed boundedness. Therefore, we derive an optimal guidance law with guaranteed airspeed boundedness.
\begin{definition}[Viability-Constrained Optimal Guidance Policy]
Given a maneuver $\sigma \in \Sigma$ with an initial airspeed $v_a \in \mathcal V$, a viability-constrained optimal guidance policy is a constant ground-referenced flight path angle command $\gamma_g:\Sigma \rightarrow \Gamma_{\mathcal{V}}$ defined over a finite horizon $[0,T(\sigma)]$ such that the induced airspeed satisfies $v_a(t) \in \mathcal V, \quad \forall t \in [0,T(\sigma)]$.
\end{definition}
Here, $T(\sigma) \triangleq \min\left(\lvert \Delta\chi \rvert / \dot{\chi}, \tau \right)$ is the maneuver horizon for turn rate $\dot{\chi} \in \mathbb{R}^+$. In this study, the turn rate is assumed constant and equal to a standard rate of 3 $^\circ$/s, so that the maneuver duration is uniquely determined by the commanded course change. A time constant $\tau \in \mathbb{R}^+$ ensures a well-defined time horizon in case of $\Delta\chi = 0^{\circ}$.

To construct viability-constrained guidance policies in practice, the time varying constraint $\Gamma_{\mathcal V}(\mathbf v_a(t),\mathbf v_w)$ is approximated through a parametric optimization framework. In particular, the maneuver space $\Sigma$ is discretized into representative combinations of course change $\Delta\chi$, wind speed $v_w$, and wind direction $\chi_w$. For each discretized maneuver $\sigma$, we solve a constrained guidance optimization that selects a constant ground-referenced flight path angle while enforcing viability and minimizing a performance cost $J \in \mathbb{R}_{\geq 0}$. The cost function is defined as:
\begin{equation}
\begin{gathered}
J(\gamma_g; \sigma)
\;\triangleq\;
\dot{v}_a^2\!\left(v_a(t),\gamma_a\right)\\
\quad t \in [0,T(\sigma)],\quad \gamma_a = \mathcal{G}^{-1}(\gamma_g;\mathbf{v}_a,\mathbf{v}_w).
\end{gathered}
\label{eq:cost}
\end{equation}
For each maneuver $\sigma \in \Sigma$, we construct a viability-constrained guidance command by solving:
\begin{equation}
\begin{gathered}
\gamma_g^\star(\sigma) =
\operatorname*{argmin}
\int_{0}^{T(\sigma)} J(\gamma_g;\sigma)\,dt \\
\text{s.t. } \gamma_g \in \Gamma_{\mathcal{V}}, \quad \forall t \in [0,T(\sigma)].
\end{gathered}
\label{eq:vc_opt_compact}
\end{equation}
The existence of an optimal guidance command for \eqref{eq:vc_opt_compact} follows from the Weierstrass extreme value theorem \cite{Arora2017-bk}.
\begin{remark}[Existence of an Optimal Solution]
For each maneuver $\sigma\in \Sigma$, the admissible set
$\Gamma_{\mathcal V} \subset \Gamma$ is assumed non-empty and compact. The objective function is time integration of acceleration, a continuous  physical quantity, therefore, continuous with respect to $\gamma_g$ on $\Gamma_{\mathcal V}$:
\begin{equation}
    \Phi(\gamma_g;\sigma) \triangleq \int_{0}^{T(\sigma)} J(\gamma_g;\sigma)\,dt.
\end{equation}
Then, by the Weierstrass extreme value theorem, $\Phi(\cdot;\sigma)$ attains a minimum on $\Gamma_{\mathcal V}$. Consequently, there exists at least one optimal solution
$\gamma_g^\star(\sigma)\in\Gamma_{\mathcal V}$ to the viability-constrained optimization problem~\eqref{eq:vc_opt_compact}.
\end{remark}

\begin{proposition}[Viability of the Optimized Guidance Policy]
\label{prop:viable_opt_policy}
Let $\sigma \in \Sigma$ and suppose $\gamma_g^\star$ satisfies \eqref{eq:vc_opt_compact}. Then the induced airspeed $v_a(t)$ satisfies $v_a(t)\in\mathcal V,\; \forall t\in[0,T(\sigma)]$ and hence $\mathcal V$ is forward invariant over the maneuver horizon.
\end{proposition}
\begin{proof}
By the constraint $\gamma_g \in \Gamma_{\mathcal{V}}$ in \eqref{eq:vc_opt_compact}, the constant command $\gamma_g^\star(\sigma)$ induces an air-referenced flight path angle $\gamma_a=\mathcal G^{-1}(\gamma_g;\mathbf v_a,\mathbf v_w),\; \forall t\in[0,T(\sigma)]$. The same constraint also enforces the tangency condition in time per \eqref{eq:viable_gamma_g_set}:
\begin{equation}
f\left(v_a,\mathcal{G}^{-1}(\gamma_g;\mathbf{v}_a, \mathbf{v}_w)\right)\in T_{\mathcal V}\left(v_a\right),
\quad \forall t\in[0,T(\sigma)].
\end{equation}
Since $\mathcal V=[v_{\min},v_{\max}]$ is compact and $f$ is continuous in its arguments, Nagumo's theorem implies that any trajectory initialized in $\mathcal V$ cannot exit $\mathcal V$ while the tangency holds. Thus, for any initial condition $v_a\in\mathcal{V}$, the induced airspeed trajectory satisfies $v_a(t)\in\mathcal{V}$ for all $t\in[0,T(\sigma)]$, establishing $\mathcal V$ as the viability domain over the maneuver horizon.
\end{proof}

\section{Numerical Realization of the Viability-Constrained Optimal Guidance Policy}
\label{sec:numeric}
This section bridges the proposed viability-assured ground-referenced guidance framework with its numerical realization. Emphasis is placed on addressing the challenges of numerical implementation in solving \eqref{eq:vc_opt_compact}. Proposition~\ref{prop:viable_opt_policy} assumes pointwise time enforcement of the viability condition, which guarantees forward invariance of the airspeed envelope at all time steps. For a nonlinear UAS model, however, this condition at every time step within an optimization loop is high overhead and can lead to and poor convergence. The time-averaged tangency condition introduced below improves numerical stability of the constraint evaluation, but the optimization still requires repeated dynamic simulation over the maneuver horizon and therefore remains computationally expensive. In the implementation, the tangency condition is therefore approximated by enforcing time-averaged airspeed accelerations at the boundaries of $\mathcal V$ over the maneuver horizon. This aggregated constraint serves as a computationally tractable surrogate that preserves the inward-pointing behavior of the airspeed dynamics in practice while relaxing strict pointwise enforcement.
\begin{definition}[Time-Averaged Tangency Condition]
\label{prop:avg_tangency}
The pointwise tangency condition
$f(v_{\min},\gamma_a) \geq 0$ and $f(v_{\max},\gamma_a) \leq 0$, $\forall t \in [0, T(\sigma)]$, is approximated by net airspeed change over the maneuver horizon per \eqref{eq:avg_tangency} where $\partial\mathcal{V}$ denotes airspeed boundaries:
\begin{equation}
\begin{gathered}
\tilde f(v_a(0),\gamma_a) \triangleq \frac{1}{T(\sigma)} \int_0^{T(\sigma)} f(v_a(t),\gamma_a)dt\\
v_a(0) = \partial \mathcal{V}.
\end{gathered}
\label{eq:avg_tangency}
\end{equation}
\end{definition}

In practice, airspeed trajectory $v_a(t)$ is obtained from a nonlinear simulation and is sampled at non-uniform time instants due to variable-step integration. To robustly approximate~\eqref{eq:avg_tangency}, the simulated trajectory is first resampled onto a uniform time grid. Let $\{(t_i, v_a(t_i))\}_{i=0}^N$ denote the simulated airspeed samples over the maneuver horizon $[0,T(\sigma)]$. A uniformly spaced time grid,
\begin{equation}
\mathcal T_\Delta
\triangleq
\left\{ \tau_k = k\Delta t \;\middle|\; k=0,\dots,K,\ \tau_K \le T(\sigma) \right\},
\end{equation}
is constructed with fixed step size $\Delta t>0$. The airspeed signal is reconstructed on $\mathcal T_\Delta$ using linear interpolation:
\begin{equation}
\tilde v_a(\tau_k)
\triangleq
\mathcal{I}\!\left(\{t_i,v_a(t_i)\}, \tau_k\right),
\end{equation}
where $\mathcal{I}(\cdot)$ denotes linear interpolation. To mitigate numerical differentiation noise arising from dynamic transients, a finite-window moving average is applied to the interpolated signal as:
\begin{equation}
\bar v_a(\tau_k)
\triangleq
\frac{1}{2L+1}
\sum_{j=-L}^{L}
\tilde v_a(\tau_{k+j}),
\label{eq:moving_average}
\end{equation}
where $L\in\mathbb N$ defines the half-width of the averaging window. The airspeed derivative is then approximated using a first-order central finite-difference scheme on the uniform grid:
\begin{equation}
\dot v_a(\tau_k)
\approx
\frac{\bar v_a(\tau_{k+1}) - \bar v_a(\tau_{k-1})}{2\Delta t}.
\label{eq:finite_difference}
\end{equation}
Finally, the time-averaged tangency condition~\eqref{eq:avg_tangency} is approximated by the discrete mean:
\begin{equation}
\tilde f\bigl(v_a(0),\gamma_a\bigr)
\approx
\frac{1}{K}
\sum_{k=0}^{K-1}
\dot v_a(\tau_k),
\label{eq:avg_tangency_discrete}
\end{equation}
which is enforced at airspeed boundaries $v_a(0)\in \partial\mathcal{V}$ within the optimization. The  smoothing operations improve numerical robustness, ensuring a reliable approximation of the inward-pointing viability condition. For sufficiently small $\Delta t$ and short averaging window $L$, linear interpolation and moving-average smoothing converge to the original signal:
\begin{equation}
    \lim_{\Delta t \to 0} \tilde v_a(\tau_k) = v_a(\tau_k),\; \lim_{L \to 0} \bar v_a(\tau_k) = \tilde v_a(\tau_k).
\end{equation}
Consequently, the discrete estimate \eqref{eq:avg_tangency_discrete} consistently approximates the continuous tangency condition \eqref{eq:avg_tangency}. While exact forward invariance of the airspeed envelope is not guaranteed in general under this relaxed enforcement, the proposed formulation provides a numerically robust and computationally efficient surrogate that preserves inward-pointing behavior in practice and enables reliable guidance-level airspeed feasibility.
\section{Application}
\label{sec:application}
This section presents the application of the proposed viability-constrained guidance framework to a 6-DoF dynamic simulation of a Cessna 182 under complete loss of thrust, serving as a representative aerial glider-like platform. 

Figure \ref{fig:flowchart} illustrates how the methodology is embedded within an autonomous UAS path planning pipeline to ensure dynamic feasibility. User-defined horizontal maneuvers, discretized over lateral motion and steady wind conditions, are first supplied to an optimization solver that utilizes the dynamic model for objective and gradient evaluations. The resulting optimal ground-referenced flight path angles $\gamma_g^\star$ are associated with their corresponding discretized maneuvers $\sigma$ to form a set of viable path planning maneuvers $\sigma_p$. A path planner then exploits this maneuver set to generate a trajectory for the autonomous UAS. These planned paths are subsequently executed in simulation, and the resulting airspeed trajectories are evaluated with respect to forward invariance of the prescribed airspeed envelope.
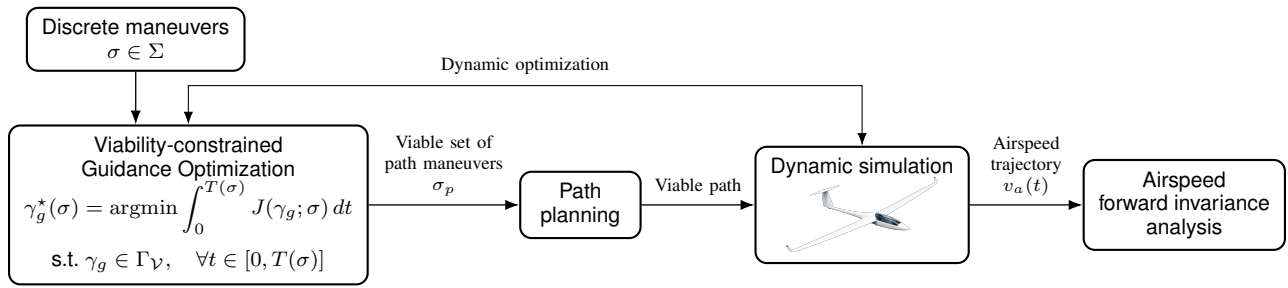
\begin{figure*}[t!]
\centering
\begin{tikzpicture}[>=Latex]

\node[block] (opt) {Viability-constrained\\ Guidance Optimization\\
$
\begin{gathered}
\gamma_g^\star(\sigma) =
\operatorname*{argmin}
\int_{0}^{T(\sigma)} J(\gamma_g;\sigma)\,dt \\[1mm]
\text{s.t. } \gamma_g \in \Gamma_{\mathcal{V}}, \quad \forall t \in [0,T(\sigma)]
\end{gathered}
$};

\node[block, above=7mm of opt, xshift=-7mm] (input) {Discrete maneuvers\\$\sigma \in \Sigma$};

\node[block, right=20mm of opt] (path) {Path\\planning};

\node[block, right=15mm of path] (sim) {Dynamic simulation\\\includegraphics[width=20mm]{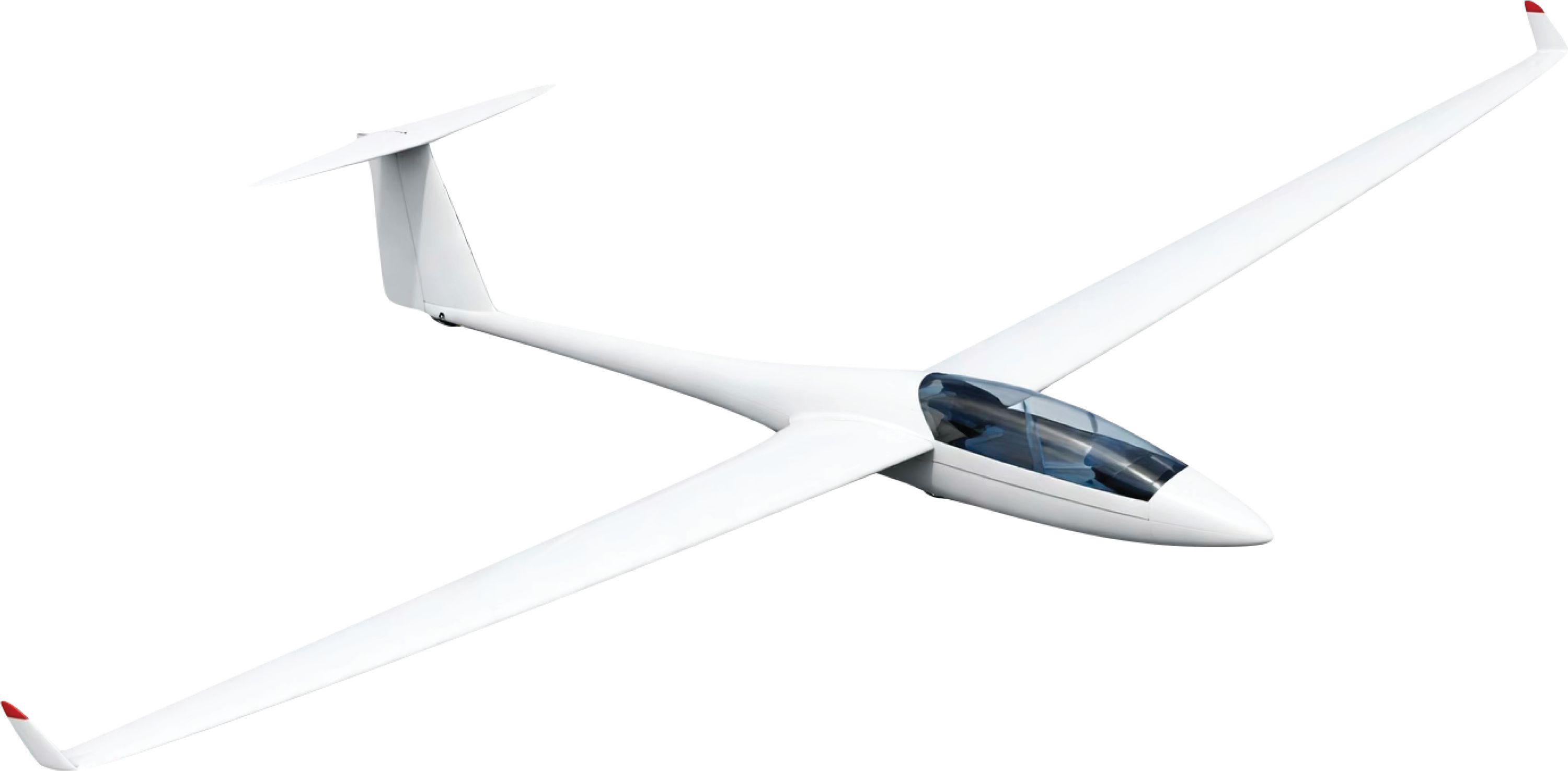}};

\node[block, right=15mm of sim] (post) {Airspeed\\forward invariance\\analysis};

\draw[<->] (sim.north) |- ++(0mm, 8mm)
  -|  node[pos=0.25, above=1pt, align=center, font=\scriptsize]{Dynamic optimization} (opt.north);

\draw[arrow] (input.south) -- ([xshift=-7mm]opt.north);

\draw[arrow](opt.east) -- node[midway, above, align=center, font=\scriptsize]{Viable set of\\ path maneuvers\\$\sigma_p$} (path.west);

\draw[arrow](path.east) -- node[midway, above, align=center, font=\scriptsize ]{Viable path} (sim.west);

\draw[arrow](sim.east) -- node[midway, above, align=center, font=\scriptsize ]{Airspeed\\ trajectory\\$v_a(t)$} (post.west);

\end{tikzpicture}
\caption{Application flowchart of the viability-constrained guidance, path planning, and dynamic simulation.}
\label{fig:flowchart}
\end{figure*}

\subsection{Scenario Description}
This section briefly introduces the path generation method used to evaluate the proposed airspeed viability framework. Although the viability-assured guidance approach is general and independent of the path planning strategy, a search-based planner for gliding UAS~\cite{tekaslan_search} is employed to generate ground-referenced paths. The planner expands the search by applying a predefined set of lateral maneuvers to generate successor states that connect an initial state to a goal through a sequence of positions. Per Figure ~\ref{fig:state_expansion}, distinct course changes applied at a parent state $s$ produce different successor states $s^j$, each generally requiring a different ground-referenced flight path angle.
\begin{figure}
    \centering
    \includegraphics[width=.8\linewidth]{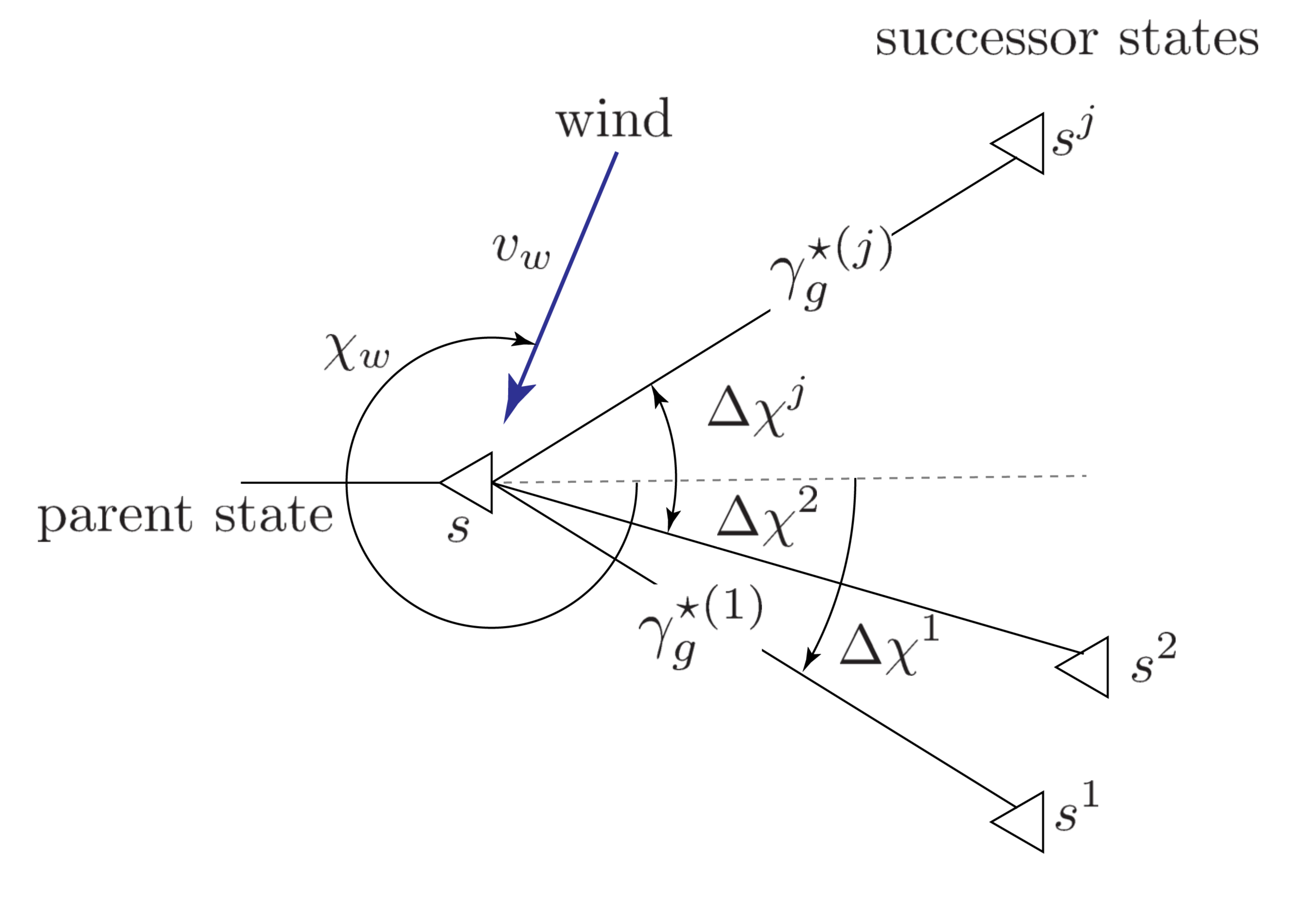}
    \caption{State expansion under steady horizontal wind.}
    \label{fig:state_expansion}
\end{figure}
For each candidate transition, the corresponding  $\gamma_g^\star$ is computed using the proposed viability-constrained guidance optimization. A complete path is then formed by concatenating these maneuver primitives.

Since each lateral maneuver $\sigma$ executed under a given wind condition and constant airspeed induces a unique $\gamma_g$, the maneuver definition is augmented for path generation. We therefore define a planner maneuver $\sigma_p$ as:
\begin{equation}
    \sigma_p \triangleq (\sigma, \gamma_g^\star)
    = (\Delta\chi, \mathbf{v}_w, \gamma_g^\star;\, \mathbf{v}_a),
\end{equation}
which fully specifies the UAS motion in both the horizontal and vertical planes through the course change $\Delta\chi$ and the associated $\gamma_g^\star$. In this way, the planner generates trajectories that are inherently compatible with airspeed-viable guidance.

\subsection{Dynamic Aircraft Model}
A 6-DoF nonlinear simulation environment is built for a Cessna 182, with stability derivatives taken from~\cite{napolitano} and the full equations of motion given in~\cite{beard_mclain}. A conventional proportional–integral–derivative attitude controller implements the flight management system. The guidance architecture employs fly-by waypoint navigation with straight-line and orbit tracking, together with half-plane waypoint switching for lateral path following~\cite{beard_mclain}. In the vertical channel, the planned optimal flight path angle serves as the reference for each trajectory segment. Complete loss of thrust is modeled by enforcing zero thrust, effectively rendering the UAS unpowered; consequently, airspeed is not directly regulated and is influenced solely through the commanded $\gamma_g^\star$.

\subsection{Numerical Setup}
A sequential quadratic programming method is used to solve the resulting nonlinear program, with Jacobians approximated via a central finite-difference scheme. The first-order optimality, constraint violation, and step tolerances are each set to $10^{-6}$, together with a fixed finite-difference step size. The optimization is evaluated over the feasible set $\Gamma = [-10^\circ,0^\circ]$ at a constant reference airspeed $v_a^\star = 90$ knots, with the viability domain defined as $\mathcal{V} = [80,100]$ knots. Course changes are sampled uniformly as $\Delta\chi \in [-90^\circ,90^\circ]$ in $15^\circ$ increments. Wind speeds are sampled at nine equally spaced values in $[0,15.6]$~knots, and wind directions are sampled uniformly over $[-\pi,\pi)$ with increments of $\pi/12$. For each combination of course change, wind speed, and wind direction, the viability-constrained optimization problem~\eqref{eq:vc_opt_compact} is solved in parallel using 36 processors over approximately three hours in MATLAB, yielding a total of 2808 distinct maneuver cases.

\subsection{Viability-constrained Guidance Policy for a Cessna 182 with Complete Loss-of-thrust}
Figure~\ref{fig:gamma_dataset} illustrates a representative subset of viability-assured maneuvers obtained by solving \eqref{eq:vc_opt_compact}. Each polar subplot corresponds to a course change $\Delta\chi$, indicated in the top-left corner, and depicts the associated $\gamma_g^\star$. The radial coordinate represents wind speed $v_w$, increasing outward from the origin, while the angular coordinate denotes the relative wind direction $\chi_w$, where $\chi_w = 0^\circ$ corresponds to a headwind and $\chi_w = \pm 180^\circ$ corresponds to a tailwind.
\begin{figure*}[t!]
     \centering
    \includegraphics[width=.9\linewidth]{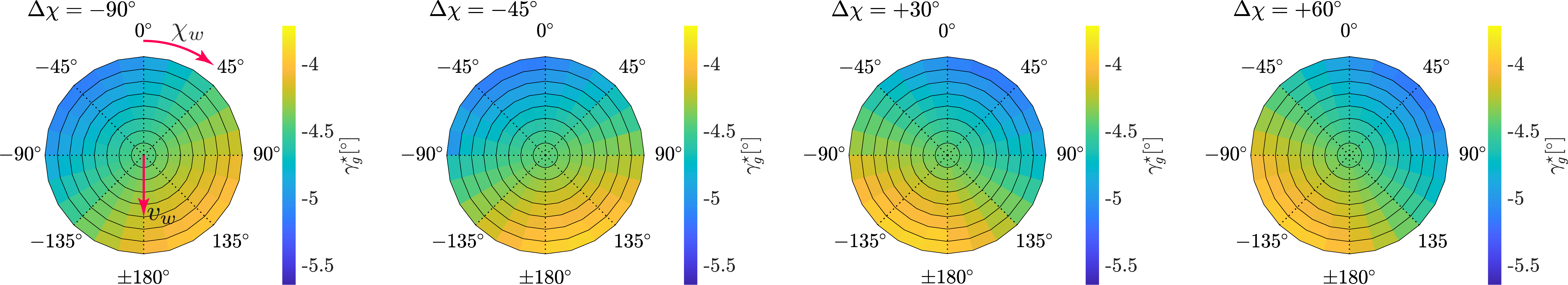}
    \caption{A subset of viability-constrained optimal ground-referenced maneuvers.}
    \label{fig:gamma_dataset}
\end{figure*}
As discussed in Section~\ref{sec:wind-coupling}, results demonstrate that maintaining airspeed boundedness during turning maneuvers under steady wind generally requires non-constant ground-referenced guidance commands. In particular, the admissible $\gamma_g^\star$ varies systematically with both wind speed and direction, reflecting the nonlinear coupling between air- and ground-referenced dynamics captured by the viability constraints.

Figure~\ref{fig:nagumo_const} shows the empirical distribution of the tangency conditions evaluated at the airspeed bounds over the synthesized maneuver set. The blue histogram corresponds to $\tilde{f}(v_{\min},\gamma_a)$, which must remain nonnegative to prevent airspeed decay at the lower boundary, while the red histogram corresponds to $\tilde{f}(v_{\max},\gamma_a)$, which must remain non-positive to prevent airspeed growth at the upper boundary. All samples satisfy the required sign conditions with nonzero margin, indicating strict satisfaction of forward invariance condition across the evaluated wind and maneuver scenarios. This result confirms that the derived guidance commands enforce airspeed boundedness robustly rather than merely satisfying the tangency condition at equality.
\begin{figure}[htb!]
     \centering
    \includegraphics[width=\linewidth]{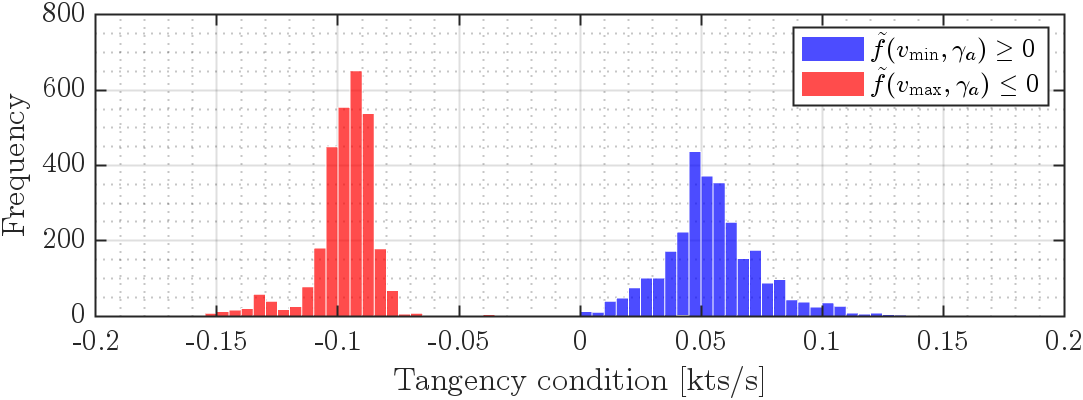}
    \caption{Empirical distribution of Nagumo tangency conditions evaluated at the airspeed bounds.}
    \label{fig:nagumo_const}
\end{figure}

\subsection{Use Case Evaluation and Airspeed Safety Benchmarking}
The optimal guidance set is numerically validated through  542 unique trajectories, totaling over 81 hours of gliding flight. Initial test altitudes differ from 3000 ft to 10,000 ft mean sea level descending to varying elevation. Top-view of all non-ascending trajectories are shown in Figure~\ref{fig:cases} where one representative case is highlighted for which a time-history of UAS dynamic response is given in Figure~\ref{fig:response}.
\begin{figure}[hbt!]
     \centering
    \includegraphics[width=0.8\linewidth]{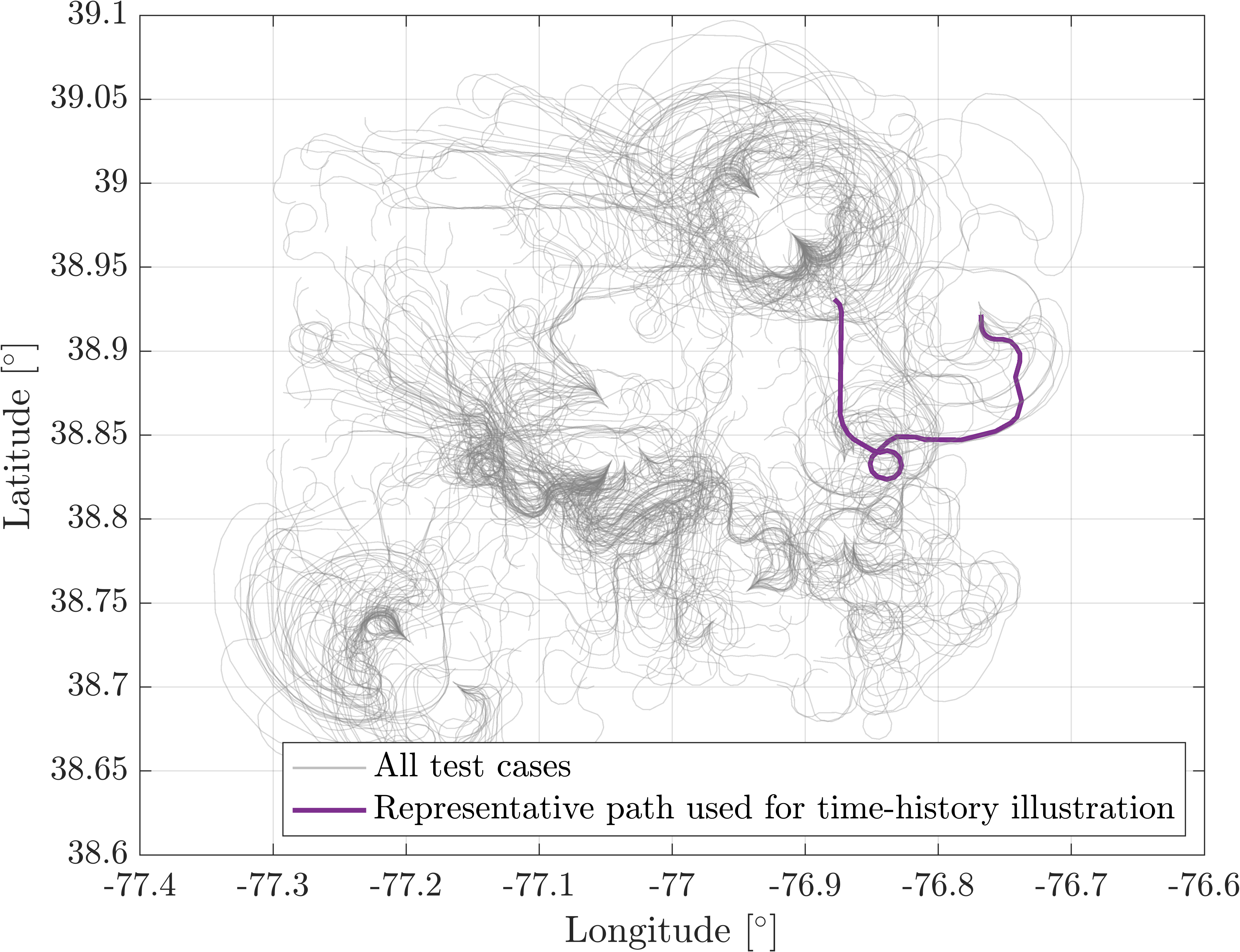}
    \caption{Top-view of all non-ascending test trajectories.}
    \label{fig:cases}
\end{figure}
The top panel shows the airspeed trajectory together with the bounds of the viable set $\mathcal V = [v_{\min},\, v_{\max}]$.
Despite frequent maneuver changes and continuously varying wind direction, the airspeed remains strictly within $\mathcal V$ over the entire horizon. The middle panel depicts the commanded and realized ground-referenced flight path angle $\gamma_g^\star$. Its variation reflects the maneuver-dependent and wind-dependent nature of the Nagumo-admissible set $\Gamma_{\mathcal V}$: for each maneuver and instantaneous wind geometry, $\gamma_g^\star$ is selected from $\Gamma_{\mathcal V}$ to ensure inward-pointing airspeed dynamics. The close agreement between command and response confirms that the viable direction is accurately realized. The bottom panel shows the course angle evolution, confirming the presence of multiple turns interspersed with straight segments. Each change in course alters the relative wind direction and hence the mapping between air- and ground-referenced flight path angles. The maintained airspeed boundedness across these transitions demonstrates that the proposed optimal ground-referenced guidance policy enforces the Nagumo condition locally in time, ensuring airspeed safety under arbitrary maneuver sequencing.
\begin{figure}[htb!]
     \centering
    \includegraphics[width=.8\linewidth]{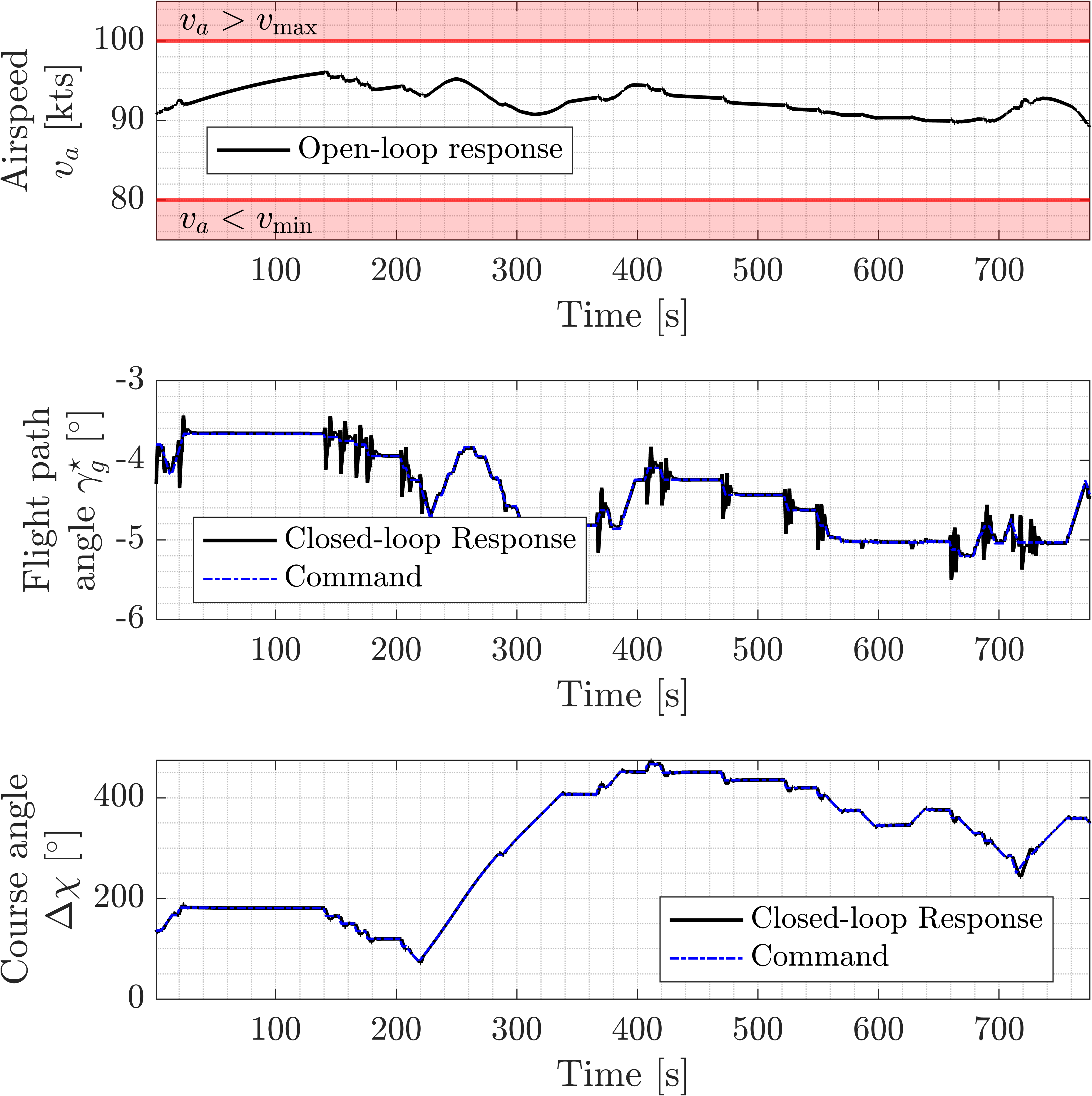}
    \caption{UAS response time history for the representative case.}
    \label{fig:response}
\end{figure}

Finally, robustness is evaluated through 542 dynamic simulations spanning a broad range of arbitrary maneuver sequences, totaling over 81 hours of cumulative gliding flight time, as shown in Fig.~\ref{fig:cases}. All scenarios are conducted under a northerly wind condition, $\mathbf{v}_w = [-15\;0\;0]^\intercal$ knots. Figure~\ref{fig:dist} summarizes the resulting airspeed behavior, providing a statistical assessment of airspeed viability under the proposed Nagumo-constrained guidance policy aggregated across all simulations and time steps.
\begin{figure}[htb!]
     \centering
    \includegraphics[width=\linewidth]{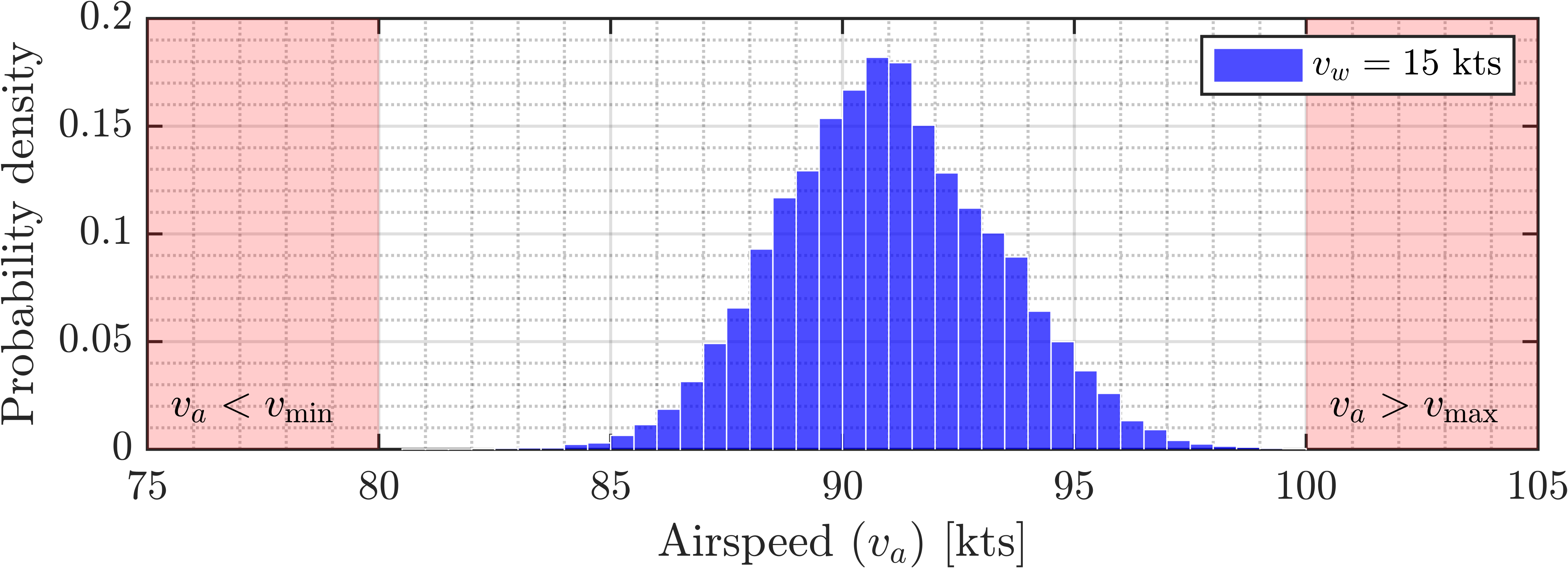}
    \caption{Airspeed distribution across all simulations and time steps, totaling over 81 hours of cumulative gliding flight data.}
    \label{fig:dist}
\end{figure}
The panel shows the empirical probability density of airspeed $v_a$. The entire probability mass lies strictly within the viability domain $\mathcal V$. Quantitatively, the absolute minimum and maximum observed airspeeds are $80.76$ and $99.98$ knots, respectively, with mean $90.99$ knots and standard deviation of $2.32$ knots.

\section{Conclusion}
\label{sec:conclusion}
This work introduced a Nagumo-based viability framework to enforce airspeed safety of unpowered fixed-wing UAS through ground-referenced guidance. The resulting admissible flight path angle policy provides a formal forward invariance guarantee for a prescribed airspeed envelope under maneuvering flight and steady wind. Simulations demonstrate that open-loop airspeed remains strictly within bounds across randomized maneuver sequences. Overall, guidance-level viability analysis can enforce airspeed safety without online control filtering, replanning, or simulation-based feasibility repair, even under strong wind. Also, while demonstrated on an unpowered UAS, the analysis applies directly to any system in which guidance commands induce safety-critical state dynamics under environmental coupling. 

Under unsteady wind, forward invariance must be enforced with respect to the instantaneous wind realization. The proposed framework remains applicable, but the admissible guidance set becomes wind-dependent, requiring the planner to account for the disturbance when generating trajectories. If the wind field is not known, conservative bounds must be enforced, which reduce the viable set and may lead to infeasible planning problems. Hardware-in-the-loop experiments and flight tests are important to evaluate the effect of model mismatch, actuator dynamics, and state-estimation errors on the admissible guidance set, and to verify that the forward-invariance guarantees derived from the model remain valid under realistic closed-loop operation. Future work will address forward invariance under non-stationary wind fields and experimental validation.

\section*{ACKNOWLEDGMENT}
This work was supported in part by NASA Langley Research Center RSES contract 80LARC23DA003.

\addtolength{\textheight}{-12cm}   

\bibliography{references}

\end{document}